\documentclass[twoside]{article}
\usepackage[accepted]{icml2024}
\input{preamble}

\usepackage{bbm}  

\begin{document}

\twocolumn[
\icmltitle{A Continuous Relaxation for Discrete Bayesian Optimization
}
\begin{icmlauthorlist}
\icmlauthor{Richard Michael}{DIKU}
\icmlauthor{Simon Bartels}{TLS}
\icmlauthor{Miguel González-Duque}{DIKU,MLLS}
\icmlauthor{Yevgen Zainchkovskyy}{DTU}
\icmlauthor{Jes Frellsen}{DTU,MLLS}
\icmlauthor{Søren Hauberg}{DTU,MLLS}
\icmlauthor{Wouter Boomsma}{DIKU,MLLS}
\end{icmlauthorlist}
\icmlaffiliation{DIKU}{DIKU, University of Copenhagen}
\icmlaffiliation{DTU}{DTU}
\icmlaffiliation{TLS}{IMT, Université de Toulouse III}
\icmlaffiliation{MLLS}{Center for Machine Learning in Life Science}

\icmlcorrespondingauthor{}{\{richard.michael ; wb\}@di.ku.dk}

\icmlkeywords{Bayesian optimization, Gaussian processes, discrete optimization}

\vskip 0.3in
]
\printAffiliationsAndNotice{}

\begin{abstract}
To optimize efficiently over discrete data and with only few available target observations is a challenge in \textit{Bayesian optimization}.
We propose a continuous relaxation of the objective function and show that inference and optimization can be computationally tractable.
We consider in particular the optimization domain where very few observations and strict budgets exist; motivated by optimizing protein sequences for expensive to evaluate bio-chemical properties.
The advantages of our approach are two-fold: the problem is treated in the continuous setting, and available prior knowledge over sequences can be incorporated directly.
More specifically, we utilize available and learned distributions over the problem domain for a weighting of the Hellinger distance which yields a covariance function.
We show that the resulting acquisition function can be optimized with both continuous or discrete optimization algorithms and empirically assess our method on two bio-chemical sequence optimization tasks.
\end{abstract}

\section{Introduction}
Optimizing discrete sequences with respect to multi-dimensional properties is particularly challenging.
If no gradient information is available and evaluation is expensive, one can use \textit{Bayesian optimization} (\bo{}) to solve the problem \citep{mockus_bayesian_1975, shahriari_taking_2016}.
The problem becomes even more challenging if our inputs are sequences of discrete elements, only few observations are given, and a strict limit is imposed on the number of possible evaluations.
Yet, most Bayesian optimization algorithms assume the optimization problem to be continuous.
Furthermore for contemporary \bo{} approaches the number of starting observations and available budget are often only tested in scenarios where these values are quite high \cite{swersky_multi-task_2013, swersky_amortized_2020}.
In this work we address the challenge of only a hand-full of available observations at the start and a strict budget \ie{}only a few hundred observations are feasible -- an \emph{ice-cold} start.
We show how to transform the optimization of discrete sequential inputs such that it is in the continuous domain, and how to utilize available probabilistic models over sequences.
Specifically, we relax the problem by mapping sequences to distributions and optimize in distribution space, which lets us incorporate prior information directly.
\begin{figure}[ht]
    \centering
\includegraphics[width=\columnwidth]{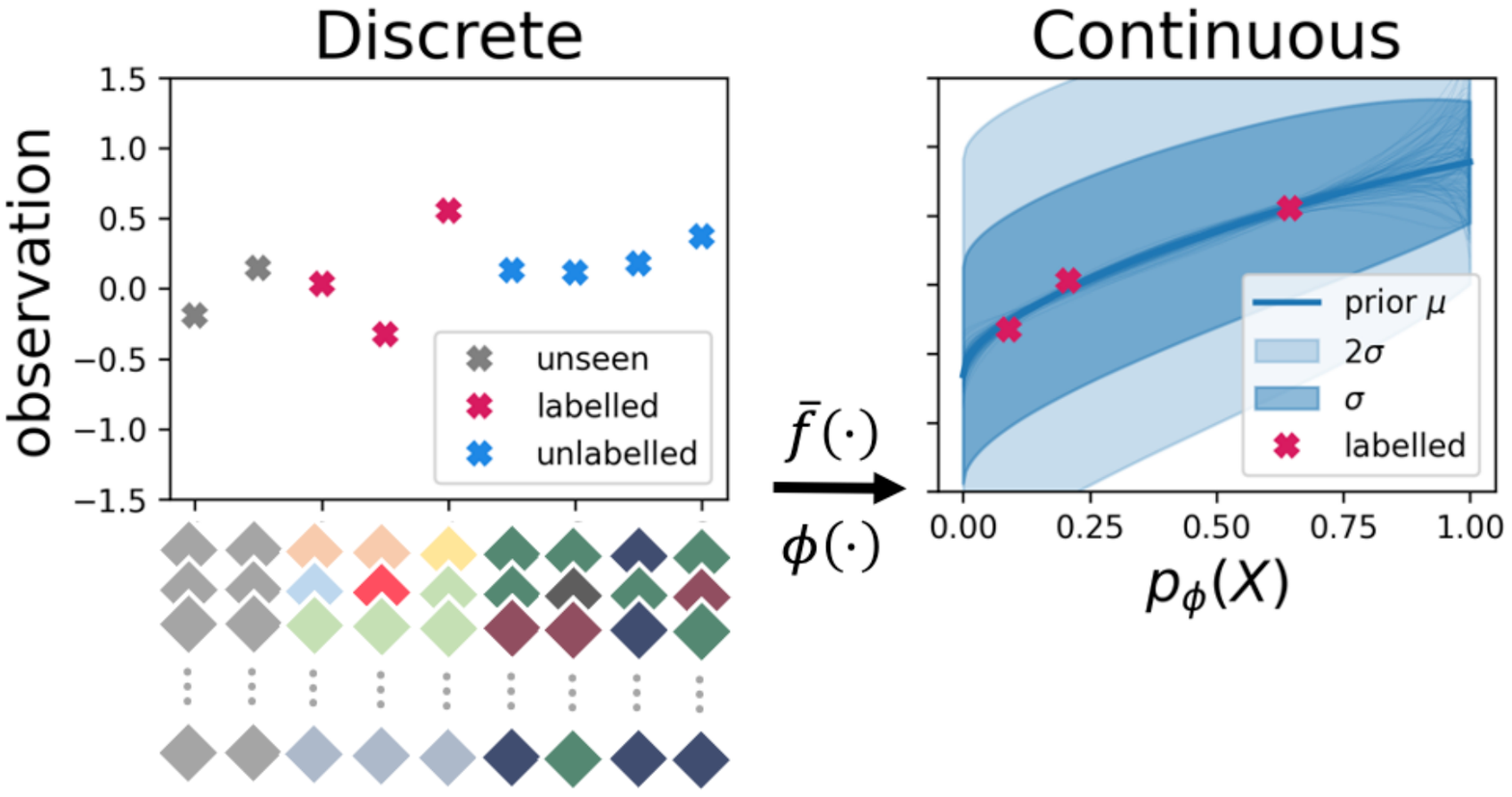}
    \caption{The proposed problem transformation: given a set of discrete element sequences ($X$ (blocks) elements in sequences) with discontinuous observations (left) we continuously relax the objective $\bar{f}$ (right), and assign a $\mathcal{GP}$ prior to it.
    The probability space over the elements is given by $\phi$, a pretrained \emph{a priori} model which parameterizes the distributions over the sequences and elements therein (middle, bottom).
    }
    \label{fig:overview}
\end{figure}

The setup of restricted discrete optimization problems is highly relevant for domains like protein engineering and drug discovery, 
where sequences of tokens are optimized, \ie{}a string of amino acids, or small molecule tokens \cite{biswas_low-n_2021, gao_sample_2022}.
Proteins are the basic building blocks of life, making their form and function the subject of scientific discovery and engineering.
To evaluate protein candidates often requires laboratory experiments making \emph{target function} values expensive to obtain.
Since proteins can be described as a sequence of amino acids from an alphabet of 20 naturally occurring ones, the search space is discrete and even for small proteins with 100 amino acids, the space of $20^{100}$ possible elements is infeasibly large \citep{maynard_smith_natural_1970}.
At the same time, only a very small subset of these sequences is likely to be useful \citep{tian_how_2017}.
These observations render both brute-force searches and direct optimization strategies for discovering useful proteins intractable and undesirable.
Previous work in optimizing bio-chemical sequences uses vast databases of known sequences to learn a low-dimensional continuous representation and perform the search for new and useful candidates therein \citep{frazer_disease_2021, notin_trancepteve_2022, stanton_accelerating_2022}.
It has been observed that in some cases, the Euclidean distance measure in learned latent representations can be a poor proxy for protein similarity \citep{detlefsen_learning_2022}, which indirectly influences Bayesian optimization techniques negatively.
This motivates us to develop an algorithm for optimizing discrete sequences directly and validate our method in the domain of protein engineering.

\textbf{Our contributions} are threefold, we \emph{first} propose a \textbf{continuous relaxation of the objective function} in the \bo{} algorithm, and \emph{secondly} show that inference and optimization remain computationally tractable. 
Specifically by proposing a covariance function that is linear in runtime with respect to input lengths.
\emph{Lastly}, we demonstrate the applicability in the domain of \emph{computational biology} by empirically assessing two discrete sequence optimization tasks.

\section{Background}

\subsection{Problem statement}
\label{sec:problem}
Given are discrete sequences of length $L\in\mathbb{N}$ and an alphabet of $A$ tokens such that, at each position of each sequence, $|A|$ elements are available.
This defines our discrete input set ${\mathbb{X}\ce \cup_{l=1}^L A^l}$.
We are given a \emph{costly} black-box function $f:\mathbb{X}\rightarrow\Re$, for which we wish to find the minimizer ${\vec x_*\ce \argmin_{\vec x\in \mathbb{X}} f(\vec x)}$ using as few function evaluations as possible.\\
For the specific case of proteins, $A$ is the set of naturally occurring 20 amino acids, then $\mathbb{X}$ contains all possible protein sequences up to a particular length $L$. 
The function $f$ may then be some measurable property of a protein system (metabolic, regulatory, etc.),
which is costly to evaluate as it requires us to run a set of wetlab experiments.

\subsection{Bayesian optimization}
Bayesian optimization typically consist of a surrogate model $m$ for $f$, and an acquisition function $\alpha$ \citep{garnett_bayesian_2022}.
The model is updated at each iteration given the observations of all experiments and the acquisition function is numerically optimized on the surrogate to select the next configuration for the evaluation of $f$.  
Typically, $m$ is a Gaussian process \citep{rasmussen_gaussian_2006}, and popular choices for $\alpha$ are \emph{Expected Improvement} \citep{jones_efficient_1998} and the \emph{Upper Confidence Bound} \citep{srinivas_gaussian_2012}. 

Since the input space is discrete, continuous numerical optimization algorithms cannot directly be used to find the optima of $\alpha$.
One possible approach to this issue is to fit a latent variable model and to perform the optimization in latent space \citep{lu_structured_2018,gomez-bombarelli_automatic_2018}.
Since it is unclear whether the Euclidean distance in representation space is a reliable proxy for similarity we choose a different approach, in the form of a continuous relaxation, a constraint probability space, and distance measure of probability vectors -- as described in section \ref{sec:c_relaxation}.

\subsection{Gaussian process regression}
Gaussian processes (GPs) are a typical choice for $m$ due to their expressiveness and closed-form inference.
Formally, a GP is a collection of random variables, such that every finite subset follows a multivariate normal distribution \citep[p.~13]{rasmussen_gaussian_2006}. 
A GP prior is described by a mean function $\mu$ (often set to $\mu(\vec x)\ce 0$), and a positive definite covariance function (kernel) $k:\mathbb{X}\times \mathbb{X}\rightarrow \Re$. 
Assuming that observations of $f$ are distorted by Gaussian noise, the posterior over the function $f$ conditioned on these observations is again a Gaussian process.

\subsection{Related work}
\citet{wan_think_2021} motivate high dimensional discrete BO as a challenge and consider a categorical and mixed search space for optimization (\emph{CASMOPOLITAN}).
\citet{stanton_accelerating_2022} proposed \emph{LaMBO}, a Bayesian optimization routine leveraging a learned lower-dimensional latent space and autoregressive model to optimize discrete sequences (small molecules and protein sequences). 
A subsequent extension of the algorithm (\emph{LaMBO}-2) still optimizes discrete space with ensemble methods replacing GPs (\citet{gruver_protein_2023}).
Further approaches conduct Bayesian optimization on structured inputs using string-based kernels \citep{moss_boss_2020} or leverage structure of the latent space i.e. for molecules \citep{gomez-bombarelli_automatic_2018, tripp_sample-efficient_2020}.

To the best of our knowledge, no current approach formulates \bo{} continuously with a \emph{covariance function on distributions} over discrete sequence inputs, defining the Gaussian Process prior over the \emph{objective} that extends to a \emph{continuous treatment of the acquisition function}.
Furthermore, obtaining a kernel by applying a weighting from  distributions has not been done prior to this work (see section \ref{sec:kernel}). 
\\
On a first glimpse, most closely related to our work may be the articles by \citet{garrido-merchan_dealing_2020} or \citet{daulton_bayesian_2022}. The former proposes a continuous relaxation for categorical inputs on discrete and mixed spaces, whereas we relax the objective.
The latter approach introduces a continuous relaxation in its \emph{probabilistic reparameterization} (PR) of the acquisition function. 
This is different from our relaxation of the objective and does not consider the constrained probability space, which we introduce in Eq. (1) and (2) respectively and which is required for computational feasability.\\

\paragraph{Issues with Gaussian process models over latent space representations}
\citet{lu_structured_2018} investigated Bayesian optimization defining a Gaussian process model directly on the latent space of a variational autoencoder (VAE).
Regular choices of covariance functions, \textit{i.e.}~the Matérn or Squared exponential kernel \citep{rasmussen_gaussian_2006} rely on a distance measure to assess covariance of two observations.
However, this property is not always satisfied in practice, since a learned latent space does not necessarily have a Euclidean measure \citep{arvanitidis_latent_2021}.\\
A key assumption for kernels based \textit{just} on Euclidean distance measures is that far-apart observations are independent from one another given a particular length-scale. 
This assumption is not necessarily fulfilled in learned latent representations, here two sequences can be highly related, while being far away in latent space or close in latent space and not at all related \citep{detlefsen_learning_2022}.

\paragraph{Issues with Bayesian optimization's budget}
Some contemporary Bayesian optimization routines rely on a significant number of black-box function evaluations \citep{maus_discovering_2023} either at the start or through a large budget of possible evaluations.
These approaches are thus limited in prohibitively expensive setting where we can only make a few hundred observations, such as bio-chemical assay experiments \cite{gao_sample_2022}.
Given a strict set-up, only few observations will be available at the start and the budget of evaluating the function limited to few expensive observations.

\section{Continuously Relaxed Bayesian Optimization}
\label{sec:idea}
For our main contribution we propose a continuous relaxation of the objective function which at a first glance appears to make the optimization problem more complicated. 
We show how the problem can be restricted and which crucial ingredients exist to solve the computational challenges induced by the relaxation.
Following the proposed steps we can optimize in our problem-space using {\it a~priori} available probability densities for our surrogate function.
This requires us to consider a covariance function that acts on probability distributions, which we develop below.

\subsection{From discrete to continuous space}
\label{sec:c_relaxation}
We can turn this discrete optimization problem (\cref{sec:problem}) into a continuous one by optimizing in the space of probability distributions over $\mathbb{X}$, minimizing the expected function value of $f$.
This way we obtain a \textbf{differentiable function} $\bar{f}$ \textbf{with the same optima} (see proof in Supplementary \cref{sec:same_optima}):
\begin{align}
    \bar{f}(\mathbf{p}):=\mathbb{E}_{\mathbf{p}}f(x)=\sum_{x\in\mathbb{X}}f(x)\mathbf{p}_x.
\end{align}
where ${\vec p \in \mathbb{P}\ce \{\vec{p}\in[0,1]^{|\mathbb{X}|}\,|\,\sum_{i}\vec p_i=1\}}$ are probability distributions over $\mathbb{X}$, which are real vectors of length $|\mathbb{X}|$ with elements between $0$ and $1$ whose components sum to $1$.
Note that, each element of $\vec x\in \mathbb{X}$ can be represented as ${\vec p_{\vec x}[\vec x']\ce \mathbbm{1}_{\vec x'=\vec x}(\vec x')}$ such that ${f(\vec x)=\bar{f}(\vec p_{\vec x})}$.

Having a continuous objective function with preserved optima, may naively appear sufficient for the successful application of \bo{}.
However, our approach is accompanied by a couple of computational challenges, which we will point out in the following.

\paragraph{Representation} 
\label{sec:representation_constraint}
Even for small proteins sequences of length less than 100, any $\vec p\in\mathbb{P}$ has more entries than there are atoms in the observable universe \citep{maynard_smith_natural_1970}. 
The space $\mathbb{P}$ will be too large to work with and we will have to restrict it.
\textbf{In the following, we will consider $\factorP$, the space of factorizing distributions of length $L$.}
\begin{align}
    \mathbb{P}_f:=\left\{\mathbf{p}\in[0,1]^{L\times A} | \mathbf{p}\geq0,\forall l: \sum_{a=1}^{A}\mathbf{p}_{l,a}=1\right\} 
\end{align}

\paragraph{Inference} Having a Gaussian process model over $f$ naturally induces a model over $\bar{f}$, by using the kernel function $k'(\vec p, \vec q)\ce \sum_{\vec x,\vec x'\in \mathbb{X}} \vec p[\vec x]k(\vec x, \vec x')\vec q[\vec x']$.
The evaluation of this canonical kernel function is intractable though, even if we consider a restriction of $\mathbb{P}$.
Approaches for this issue are presented in \cref{sec:model_bar_f}.

\paragraph{Optimization} Having established a model $m$ over $\bar{f}$ we must also address the problem of how to optimize a Bayesian optimization's acquisition function $\alpha_{m}$. 
Even though this acquisition function is continuous, the fact that proposed inputs must remain probability distributions prevents us from freely using any optimizer\footnote{The vector components have to be positive and sum to one.}.
This topic is discussed in \cref{sec:optimization}.

\subsection{The model}
\label{sec:model_bar_f}
A Gaussian process prior over $f$ induces a Gaussian process belief over $\bar{f}$, yet the computation of the posterior over $\bar{f}$ is intractable -- even if we restrict ourselves to $\factorP$.
The \textbf{key idea is to place a Gaussian process prior directly over $\bar{f}$} instead of using the induced prior from $f$.
Doing so will allow us to do computationally tractable inference.
We note that in general it is \emph{not} the case that $\bar{f}(\vec p)=\sum_{x}\vec p(x)\bar{f}(\mathbbm{1}_x)$, however we can address this in practice. 
The main challenge is to find a kernel function which can exploit the structural properties of $\mathbb{P}_f$; for example distances informed by {\it prior} probability densities, and subsequently discard regions of low-probability.

\subsubsection{The weighted Hellinger kernel}
\label{sec:kernel}
One such kernel ${k(\cdot , \cdot): \mathbb{P}_f\times\mathbb{P}_f\rightarrow \mathbb{R}}$ is based 
on the \emph{Hellinger} distance $r$ \citep{hellinger_neue_1909}: 
\begin{align}
    r(\mathbf{p},\mathbf{q})&:=\sqrt{\frac{1}{2}\sum_{x\in\mathbb{X}}\left(\sqrt{\mathbf{p}(x)}-\sqrt{\mathbf{q}(x)}\right)^2}, \label{eq:hellinger}\\
    k(\mathbf{p},\mathbf{q})&:=\theta\exp(-\lambda r(\mathbf{p},\mathbf{q})).
\end{align}

We know $r$ to be negative definite \cite{harandi_beyond_2015}, and therefore $k$ is a \emph{positive definite} kernel for all $ \sigma,\lambda > 0 $ \citep{feragen_geodesic_2015}.
Since we restrict $\vec p, \vec q$ to be elements of $\factorP$, we can evaluate in $k(\vec p, \vec q)$ in $\mathcal{O}(L\times A)$ time by rewriting \cref{eq:hellinger} as:
\begin{align*}
    r^2(\mathbf{p},\mathbf{q})&=1- \prod_{l=1}^{L}\sum_{a=1}^{A}\sqrt{\mathbf{p}[l,a]\mathbf{q}[l,a]}.
\end{align*}
For the proof see \cref{thm:efficient_hd} in the supplementary.

This kernel will not be useful to guide optimization which becomes apparent when considering the Hellinger distance for any distinct sequences ${\vec x,\vec x'}$: ${r(\mathbbm{1}_{\vec x}, \mathbbm{1}_{\vec x '})=1}$. 
However, very often there exists a \emph{prior ranking over elements of $\mathbb{X}$} in form of a probability distribution.
For example, sequence alignment algorithms used to search biological databases often return a hidden Markov model, or a variational autoencoder can be used to compute likelihoods \citep{durbin_biological_1998, riesselman_deep_2018, frazer_disease_2021}.
Such models are trained in an unsupervised way and can correlate sufficiently with $f$ \citep{frazer_disease_2021}. 
To make use of this prior knowledge, we propose to weigh the Hellinger distance using the given ranking.
For every positive weighting $w:\mathbb{X}\rightarrow\mathbb{R}_+$, we define the weighted Hellinger distance as 
$$r_w^2(\vec p, \vec q)\ce \frac{1}{2}\sum_{\vec x\in\mathbb{X}} w(\vec x)\left(\sqrt{p(\vec x)}-\sqrt{q(\vec x)}\right)^2\space.$$
The proof that $r_w$ still gives rise to a kernel function can be found in the supplementary material (\cref{thm:weighted_hd}).
When considering distinct sequences $\vec x, \vec x'$, the weighted Hellinger distance evaluates to ${r_w(\mathbbm{1}_{\vec x}, \mathbbm{1}_{\vec x'})=\frac{1}{2}w(\vec x)+w(\vec x')}$.
This means that sequences which have low weighting are considered similar whereas sequences with high weighting are considered more independent.

The kernel is particularly suitable to model functions that have a cut-off or where the majority of inputs yields a particular value.
If the weighting is 0 for sequences that have the same function-value, they correlate perfectly under this kernel, allowing us to disregard this vast space with one function evaluation in our \bo{} routine.

When choosing $w$ it is important to allow for efficient evaluation\footnote{Note, $w$ can be of different form than elements of $\factorP$.} of $r_w$.
In the following we will consider two options: i) hidden Markov models \citep{durbin_biological_1998} and ii) variational autoencoders (VAEs) \citep{kingma_introduction_2019}.
\cref{fig:hellinger_contour_zero} shows a visualization of how the combination of Hellinger kernel and decoder induces a more complex, non-Euclidean similarity measure on the latent space.

\begin{figure*}[!h]
\centering
\includegraphics[width=0.8\textwidth]{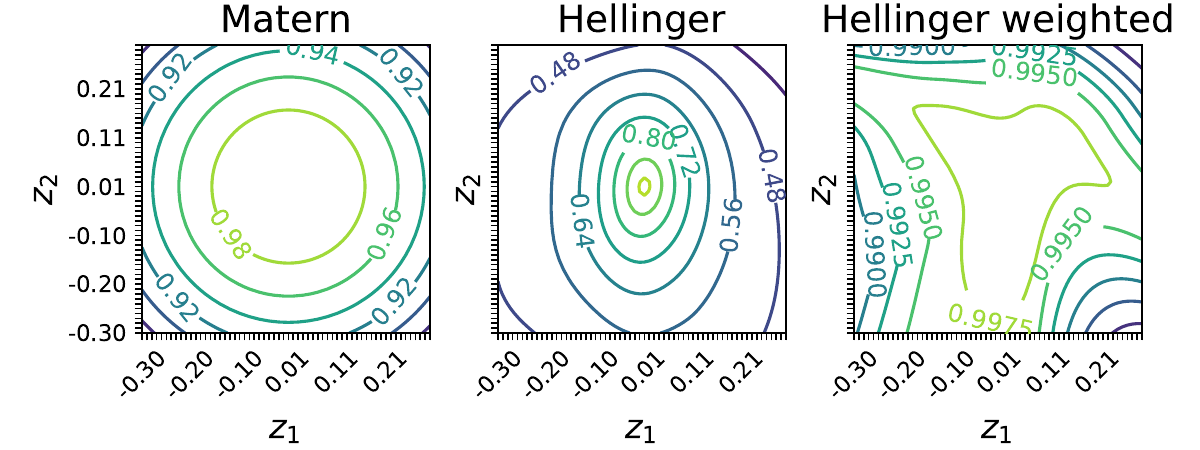}
\caption{Visual comparison of the (weighted) Hellinger distance kernel and a Mat\'ern 5/2.
We are using a two-dimensional version of the decoder proposed by \citet{brookes_conditioning_2019}.
For the Mat\'ern kernel we are visualizing $k(0, [z_1, z_2])$ whereas for the Hellinger kernel we show $k(P(x\g 0), P(x'\g [z_1,z_2]))$.
With the Hellinger kernel, the decoder induces a much more complex, non-Euclidean similarity measure on the latent space. 
Note that, the Hellinger kernels are non-stationary in the latent space---\cref{fig:hellinger_contour_gfp}
in the supplementary displays the same visualization for a different reference point.
}
\label{fig:hellinger_contour_zero}
\end{figure*}

\subsubsection{Modeling details}
This seemingly simplistic kernel gives rise to expressive Gaussian process models by recalling that products of kernels define a valid kernel \citep[p.~95]{rasmussen_gaussian_2006}.
More specifically, for a given latent variable model we propose to use $k'(x_A,x_B)\ce \prod_{n=1}^N k_{P(X|Z_n)}(x_A, x_B)$
with $Z_n$ being subsets of the latent space and the same hyper-parameters $\theta$ and $\lambda$ for each kernel. 
This is a product kernel over the available samples.
We defer specifying our choices for $Z_n$ to \cref{sec:experiments}.

Regarding the hyper-parameters, we 
follow \citet{jones_efficient_1998} and set a constant prior mean function with ${\mu\ce \frac{\vec 1\Trans \inv{\mat K}\vec y}{\vec y\Trans\inv{\mat K}\vec y}}$ and the amplitude of the kernel to ${\theta\ce \frac{1}{N-1}(\vec y - \mu)\Trans\inv{\mat K}(\vec y - \mu)}$.
The other parameters, which are $\lambda$ and the observation noise $\sigma^2$, we set by maximizing the evidence {$\log p(\vec y|\lambda,\sigma^2)$}.

\subsection{Optimizing the acquisition function}
\label{sec:optimization}
As a consequence of the continuous relaxation, the acquisition function also acts on the probability distributions.
Essentially, there are three ways to find local optima:
(i) discrete optimization algorithms, (ii) using a continuous parameterization of $\factorP$, (iii) manifold optimization on $\factorP$, and (iv) any combination of the former options.

Note that our acquisition function still allows us to rate sequences individually by treating them as indicator functions.
Therefore, any discrete optimization algorithm applies, which we use later to compare with existing discrete optimizers.

Another option for optimization arises if we have a continuous parameterization for $\factorP$, for example given through the decoder of a VAE $\operatorname{Dec}: \Re^d \rightarrow \factorP$.
In that case, we can apply standard continuous optimization algorithms to ${\beta: \Re^d\rightarrow \Re, \vec z \mapsto \alpha(\operatorname{Dec}[\vec z])}$.

In the ideal case, the optimization stops with an atomic distribution\footnote{A deterministic distribution in which one outcome has a likelihood of one and all others zero.} such that the decision which sequence to evaluate is unambiguous.
Even if this is not the case, the optimization helps to narrow down the choice of candidates.
To decide on a final candidate, we may choose the most likely sequence or we can sample from the optimized distribution and use the same acquisition function to score the sampled candidates.
The latter option makes our approach applicable for a batched optimization setting.

A third option to optimize the acquisition function, we mention here in passing, since its implications are out of the scope of this work.
The space $\factorP$ is the product simplex, a manifold.
Hence, we can apply any manifold optimization algorithm directly on the acquisition function \cite{boumal_optimization_2014}.

\renewcommand{\algorithmicrequire}{\textbf{Input:}}%
\renewcommand{\algorithmicensure}{\textbf{Output:}}%

\subsection{The \CoRel{} algorithms}
\cref{algo:general_bo} shows how \CoRel{} can be used in the general Bayesian optimization loop given a continuous parameterization of $\factorP$, here the decoder of a variational autoencoder.
See \cref{algo:si:continuous_bo} and \cref{algo:si:discrete_bo} in the supplementary material for the conceptual difference when using discrete optimization or direct continuous optimization.

Given is a set of starting sequences with observations, the acquisition function on the $\factorP$ space, and a pretrained model with which we derive a probability vector for each sequence. 
We fit a GP posterior predictive optimizing the kernel parameters. 
The acquisition function is queried for the maximizing point on the parameterized space via the predictive GP to find the maximizing probability vector.
From the probability vector we can obtain a sequence or a set of sequences within a given budget.
The black-box function is evaluated on each sequence and the observations as well as the inputs are added to the dataset. 
These steps are repeated until the budget is exhausted.

\begin{algorithm}
\caption{\CoRel{} using parameterized optimization}\label{alg:cap}
\begin{algorithmic}
\Require acquisition $a:\factorP\rightarrow\Re $, black-box $f:\mathbb{X}\rightarrow \Re $, dataset $\mathcal{D}_1=\{X, y\}$, pretrained LVM $\phi: \Re^D\rightarrow \factorP$
\For{{$t \in 1, ..., t_{\max}$}}
\State $m\gets$ trainModel($(\mathbbm{1}_{\vec x_i}, \vec y_i)_{i=1}^t$) 
\State $\vec z_* \gets \arg\max_{\vec z}a(\phi(\vec z), m)$ 
\State $\vec p_* \gets \phi(\vec z_*)$
\State $\vec x\gets$ getSequenceFromDistribution($\vec p_*$)
\State $\mathcal{D}_{t+1} \gets \mathcal{D}_{t} \cup \{\vec x, f(\vec x)\}$
\EndFor
\end{algorithmic}
\label{algo:general_bo}
\end{algorithm}

\begin{algorithm}
\caption{getSequenceFromDistribution}
\begin{algorithmic}
\Require distribution $P$, acquisition function $a$, budget b
\Ensure $\vec x_*$
\State $\vec x_*=\arg\max_{\vec x}P(\vec x)$  
\State $y_* = a(\mathbbm{1}_{\vec x_*})$
\For{$t \in 1, ..., b$}
\State $\vec x\sim P$
\State $y = a(\mathbbm{1}_{\vec x})$
\If{$y>y_*$}
\State $y_* = y$, $\vec x_*=\vec x$
\EndIf
\EndFor
\end{algorithmic}
\label{algo:distribution_subroutine}
\end{algorithm}
Now that we have a complete formulation of the proposed Bayesian optimization procedure we assess the results empirically.

\begin{figure}[!h]
\includegraphics[width=\columnwidth]{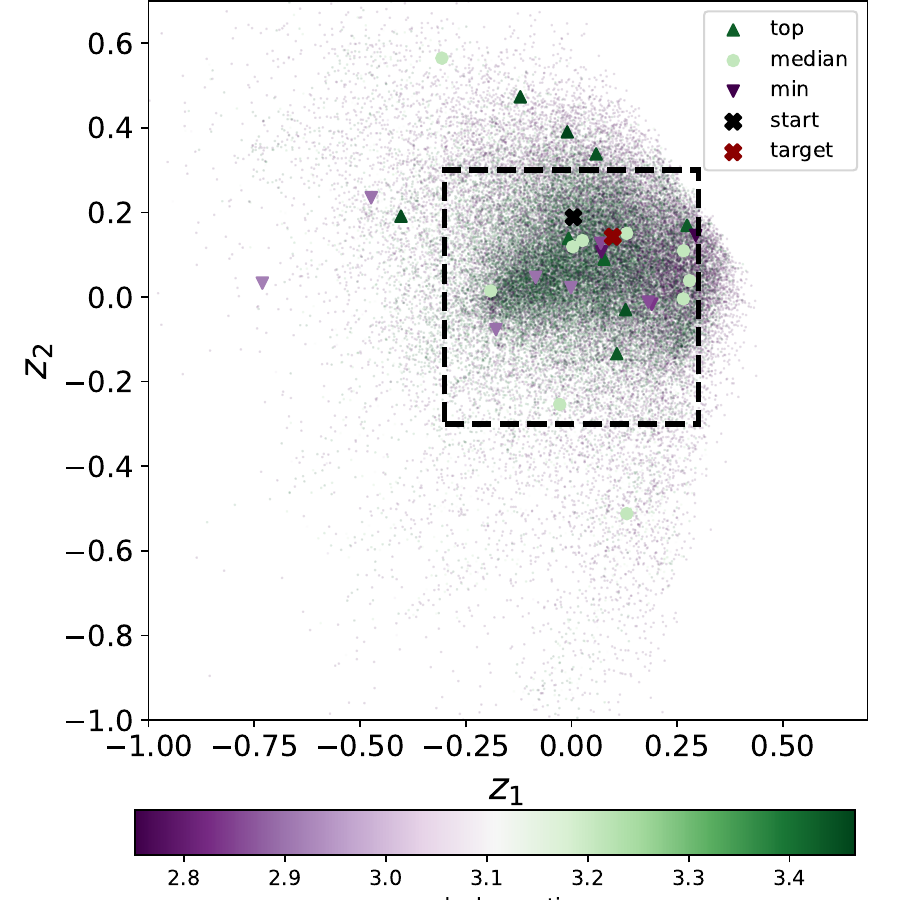}
    \caption{Two dimensional VAE latent space adapted from \cite{brookes_conditioning_2019}.
    We encode the \GFP{} corpus of experimentally evaluated sequences (dots). Available for optimization are only oracle evaluations - see markers for oracle predictions (top \ding{115}, median \ding{108}, and lowest \ding{116} 10 observations each).
    Start is the reference wild-type and target the maximally fluorescent candidate. 
    }
    \label{fig:latent_space}
\end{figure}

\section{Empirical Results}
\label{sec:experiments}
Two problems serve as an empirical benchmark of our method: (i) \citet{stanton_accelerating_2022} optimized several \emph{Red Fluorescent Protein}s (\RFP{}) with \emph{LamBO} and (ii) \citet{brookes_conditioning_2019} defined a model to optimize \emph{enhanced Green Fluorescent Protein} (e\GFP{}) with \emph{CBas}.
We optimize both problems with \CoRel{} in the \emph{ice-cold} setup, where only the candidates are available that strictly define the problem; specifically 6 \RFP{} and 3 \GFP{} sequences.
This means that only the observations for initial candidate sequences are available which are to be optimized and no other measurements can be obtained initially.
A setup particularly relevant for drug-discovery where initial experimental observations can be prohibitively expensive \citep{dubach_quantitating_2017}.

In order to systematically test the protein sequences for their properties under different experimental setups we provide \href{https://machinelearninglifescience.github.io/poli-docs/}{poli} -- a dedicated protein optimization library to query black box functions with discrete sequence inputs.

\subsection{Continuous optimization}
\paragraph{Optimizing with a Latent Variable Model} (LVM)
is arguably the most appealing approach for \CoRel{} as a continuous optimizer.
In this section we investigate the \GFP{} problem for which a continuous parameterization of $\factorP$ already exists as a pretrained latent decoder
\citep{brookes_conditioning_2019} (see Supplementary \cref{sec:cbas_setup} for details).
Figure \ref{fig:latent_space} shows an adaptation of the \GFP{} VAE as a learned two-dimensional embedding, for which only an oracle function can be queried (see markers) and experimental observations are unavailable.
The CBas oracle function evaluations serve as surrogate for the true \GFP{} fluorescence values. 
To qualitatively inspect our method we evaluate the covariance function values of the (2D) latent space in an area around the reference sequence. 
\cref{fig:hellinger_contour_zero} shows the evaluated Hellinger (\cref{eq:hellinger}) and weighted Hellinger kernel, which are \emph{not} equidistant in latent space (like the Matérn) respective a reference point while the decoder indeed informed the covariance function values.
These values change as anticipated when computed with respect to a different reference point (see Supplementary \cref{fig:hellinger_contour_gfp}). 
In particular higher covariance values are assigned in a density around the reference points and the respective decoding probabilities.

When conducting optimization we observe larger objective values within the allotted budget (100 queries) compared to a random acquisition while the model appears to prioritize extreme values in the problem-space (see Supplementary  \cref{fig:si:gfp_obs}).

\subsection{Discrete optimization}
\paragraph{Optimizing in the discrete proposed sequence setting} is what contemporary optimization protocols often translate to if the generated proposals are directly in the sequence space. 
Therefore we show that our approach is still applicable to the discrete setting.
For this we optimize the \RFP{} problem with respect to multiple properties: stability and surface area accesibility (SASA).
We use LamBO as a SoTA reference in this optimization setting and compare also against random sequence mutations.
Again we focus in particular on a strict setup where only the Pareto front is given (6 reference \RFP{} sequences) and the number of oracle evaluations is limited to 180 queries specifically.
To optimize for multiple tasks in the \bo{} algorithm we use
the expected hypervolume improvement (EHVI) as an acquisition function \citep{daulton_differentiable_2020}.

To build the $\mathbb{P}_f$-space we rely \emph{only} on a Hidden Markov Model (HMM), obtained by the standard HMMER algorithm \cite{eddy_accelerated_2011}. 
We use this as $\phi$ model to parameterize our distributions.
The choice for an HMM is natural, as it is built in the process of querying for related sequences to the starting candidates, a step that would be required when setting up the initial \RFP{} problem-set \citep{stanton_accelerating_2022}\footnote{The LamBO \RFP{} data includes a wide range of additional RFP sequences (see Supplementary \cref{sec:lambo_setup}).}. 

\cref{fig:discrete_opt} shows that running \CoRel{} obtains a larger relative hypervolume compared to SoTA method \emph{LamBO}, and a simple random mutation baseline which modifies 2 amino acids selected at random from elements in the Pareto front and retains the best results for subsequent iterations. 
This results in a larger Pareto front of the respective protein candidates (see \cref{fig:discrete_opt_front}).

Our results also show that we achieve on-par performance in the reference setup  where we have just over 500 of initial sequences available (see Supplementary \cref{fig:si:ref_HV}).
Under those conditions we outperform existing methods with respect to the achieved (relative) hyper-volume within the given budget - see \cref{fig:discrete_opt}.
However we note that if multiple starting candidates are available (\ie{}\ $N=50$) a larger Pareto-front is optimized and \emph{LamBO} does outperform our method.

\begin{figure*}[ht]
    \centering
    \begin{minipage}[b]{0.44\linewidth}
    \includegraphics[width=\columnwidth]{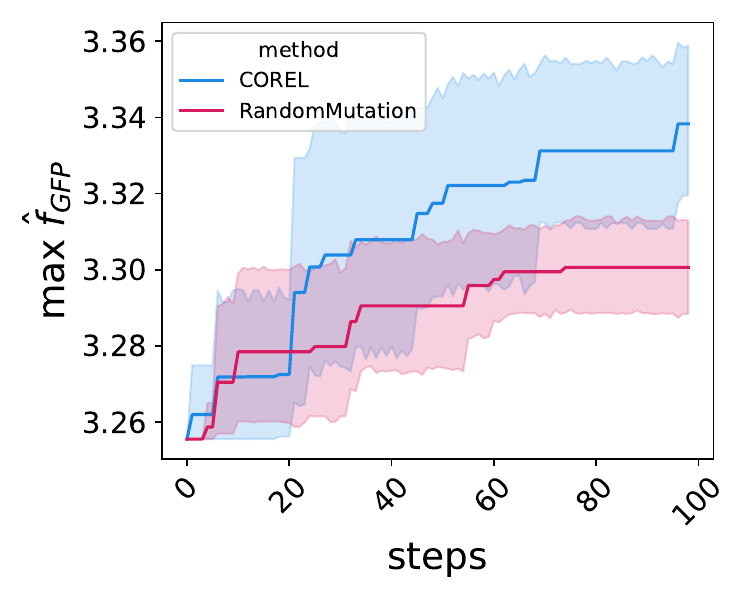}
    \caption{The \GFP{} sequences are optimized by \CoRel{} and random mutations with the CBas oracle ($\hat{f}_{\text{GFP}}$) over 100 steps.
    We observe best sequences selected at an iteration with mean (line) and 95\%CI (shaded) across seven seeds.
    }
    \label{fig:continuous_opt}
    \end{minipage}
    \hfill
    \begin{minipage}[b]{0.49\linewidth}
    \centering
    \includegraphics[width=\columnwidth]{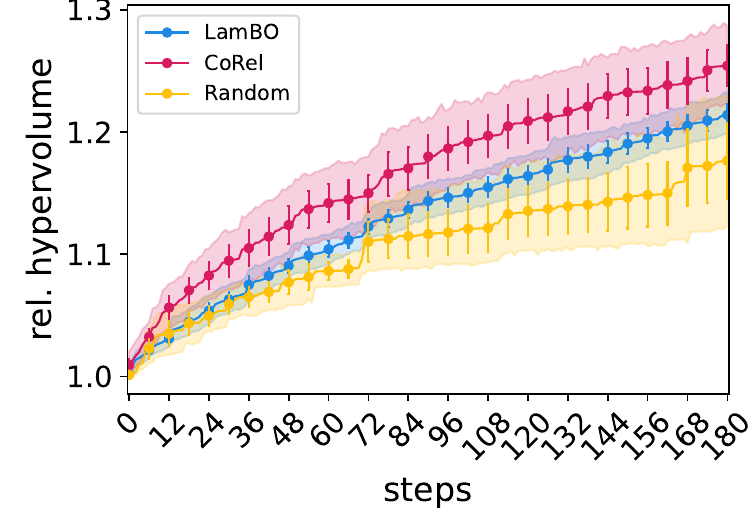}
    \caption{
    The RFP Pareto front is optimized discretely computing relative hypervolume respective the starting sequences, comparing CoRel with LamBO starting with six observations. 
    Markers ({\tiny \ding{108}}) indicate batch average and std.err. bars across 21 seeds (random 5).
    }
    \label{fig:discrete_opt}
    \end{minipage}
\end{figure*}

\begin{figure}[ht]
    \centering
    \includegraphics[width=0.8\columnwidth]{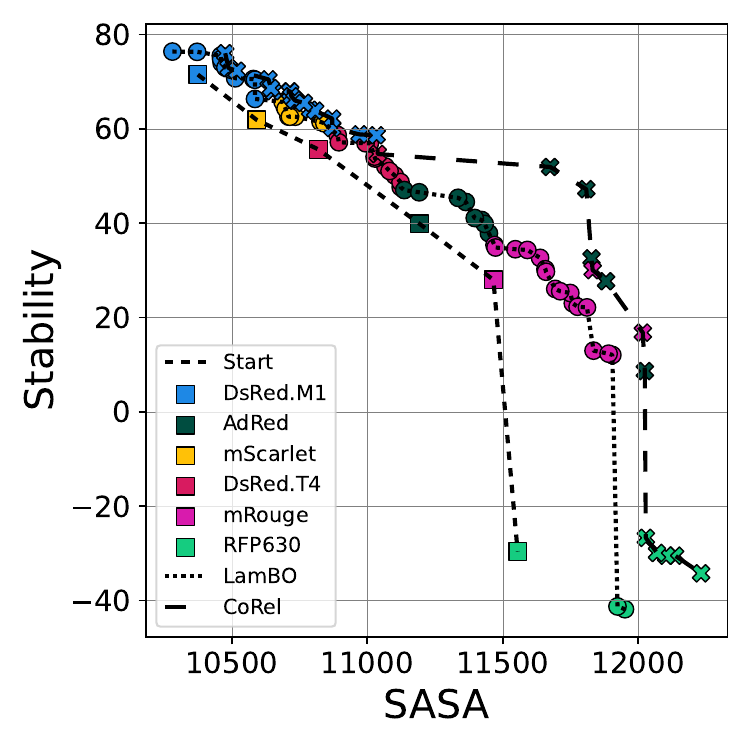}
    \caption{
    Pareto front of the RFP properties stability and surface accessibility (SASA).
    The \emph{start} Pareto sequences (\ding{110} color-coded), the best variants proposed by LamBO (\ding{108}) and CoRel (\ding{54}) after 32 iterations (21 seeds) from a cold start.
    }
    \label{fig:discrete_opt_front}
\end{figure}

\section{Discussion}
\paragraph{We focus on the surrogate model rather than the acquisition function}
Advances in Bayesian optimization can either be done by building a suitable surrogate function to model the unknown true $f$ or by investigating the acquisition function.
Our contribution is apart from the problem transformation primarily the surrogate model via the covariance function.
\CoRel{} focuses on the properties of the surrogate function - in particular, a GP which incorporates a ranking over the continuously relaxed inputs.
This yields a surrogate model predicting points of interest to observe, at the cost of capturing the underlying function landscape.
While the acquisition function is an interesting topic of research in its own right it is out of the scope of this work and we rely on the established results in the field \cite{jones_efficient_1998}.

\paragraph{Relying on established prior models can alleviate model complexity}
Our setup of continuously relaxing and using a pre-computed model such as HMMs parameterizing our constrained probability space is different compared to contemporary latent space optimizations.
Given that distributions or a parameterization of elements of a sequence are available for a particular problem-domain that have been shown to work reliably, they can be used directly in the \CoRel{} approach.
This alleviates to some extent to account for a plethora of potential decisions which are required for latent-space BO approaches: which encodings to use, which surrogate model to train given the trained latent representation, the sampling and reweighing of base and proposal sequences, etc. . 

\paragraph{Simple prior models may impact performance}
The results show that we perform SoTA in particular in the cold-start setting, restricted domain - relevant for \bo{} on bio-chemical sequences.
The performance in less restrictive conditions can be improved by tuning the models which parameterize the densities.
The contributions of this paper are mainly the problem transformation and its application to sequence optimization.
Improving upon the pretrained models for which the weighted Hellinger kernel can be applied will be left for future work.

\section{Conclusion}
We have shown Bayesian optimization by continuous relaxation with the proposed \CoRel{} approach, namely a continuous relaxation of the objective and a Gaussian process acting on sequences as probability distributions.
We are able to do optimization with the introduced weighted Hellinger kernel after having shown that it is indeed a kernel.
We have demonstrated the applicability in the bio-chemical domain and compared empirically with contemporary methods by optimizing various protein sequences.
Here we have demonstrated feasibility and performance particularly in significantly more challenging formulations of the optimization problems.

\section{Software and Data}
We provide the code for \CoRel{} on \href{https://github.com/MachineLearningLifeScience/corel}{GitHub} as well as the \href{https://github.com/MachineLearningLifeScience/poli}{\emph{poli} package on}, both under the MIT license.

\section*{Acknowledgements}
This work was in part funded by the Danish Data Science Academy, which is funded by the Novo Nordisk Foundation (NNF21SA0069429) and VILLUM FONDEN (40516).
Further funding includes the MLLS Center (Basic Machine Learning Research in Life Science NNF20OC0062606), the Danish Ministry of Education and Science, Digital Pilot Hub, Skylab Digital, and the Pioneer Centre for AI (DRNF grant number P1).

\bibliography{references}

\begin{thebibliography}{41}
\providecommand{\natexlab}[1]{#1}
\providecommand{\url}[1]{\texttt{#1}}
\expandafter\ifx\csname urlstyle\endcsname\relax
  \providecommand{\doi}[1]{doi: #1}\else
  \providecommand{\doi}{doi: \begingroup \urlstyle{rm}\Url}\fi

\bibitem[Arvanitidis et~al.(2021)Arvanitidis, Hansen, and
  Hauberg]{arvanitidis_latent_2021}
Arvanitidis, G., Hansen, L.~K., and Hauberg, S.
\newblock Latent {Space} {Oddity}: on the {Curvature} of {Deep} {Generative}
  {Models}, December 2021.
\newblock URL \url{http://arxiv.org/abs/1710.11379}.
\newblock arXiv:1710.11379 [stat].

\bibitem[Biswas et~al.(2021)Biswas, Khimulya, Alley, Esvelt, and
  Church]{biswas_low-n_2021}
Biswas, S., Khimulya, G., Alley, E.~C., Esvelt, K.~M., and Church, G.~M.
\newblock Low-{N} protein engineering with data-efficient deep learning.
\newblock \emph{Nature Methods}, 18\penalty0 (4):\penalty0 389--396, April
  2021.
\newblock ISSN 1548-7105.
\newblock \doi{10.1038/s41592-021-01100-y}.
\newblock URL \url{https://www.nature.com/articles/s41592-021-01100-y}.
\newblock Number: 4 Publisher: Nature Publishing Group.

\bibitem[Boumal(2014)]{boumal_optimization_2014}
Boumal, N.
\newblock Optimization and estimation on manifolds.
\newblock 2014.
\newblock Publisher: Catholic University of Louvain, Louvain-la-Neuve, Belgium.

\bibitem[Brookes et~al.(2019)Brookes, Park, and
  Listgarten]{brookes_conditioning_2019}
Brookes, D., Park, H., and Listgarten, J.
\newblock Conditioning by adaptive sampling for robust design.
\newblock In \emph{Proceedings of the 36th {International} {Conference} on
  {Machine} {Learning}}, pp.\  773--782. PMLR, May 2019.
\newblock URL \url{https://proceedings.mlr.press/v97/brookes19a.html}.
\newblock ISSN: 2640-3498.

\bibitem[Daulton et~al.(2020)Daulton, Balandat, and
  Bakshy]{daulton_differentiable_2020}
Daulton, S., Balandat, M., and Bakshy, E.
\newblock Differentiable {Expected} {Hypervolume} {Improvement} for {Parallel}
  {Multi}-{Objective} {Bayesian} {Optimization}.
\newblock In \emph{Advances in {Neural} {Information} {Processing} {Systems}},
  volume~33, pp.\  9851--9864. Curran Associates, Inc., 2020.

\bibitem[Daulton et~al.(2022)Daulton, Wan, Eriksson, Balandat, Osborne, and
  Bakshy]{daulton_bayesian_2022}
Daulton, S., Wan, X., Eriksson, D., Balandat, M., Osborne, M.~A., and Bakshy,
  E.
\newblock Bayesian {Optimization} over {Discrete} and {Mixed} {Spaces} via
  {Probabilistic} {Reparameterization}, October 2022.
\newblock URL \url{http://arxiv.org/abs/2210.10199}.
\newblock arXiv:2210.10199 [cs, math, stat].

\bibitem[Detlefsen et~al.(2022)Detlefsen, Hauberg, and
  Boomsma]{detlefsen_learning_2022}
Detlefsen, N.~S., Hauberg, S., and Boomsma, W.
\newblock Learning meaningful representations of protein sequences.
\newblock \emph{Nature Communications}, 13\penalty0 (1):\penalty0 1914, April
  2022.
\newblock ISSN 2041-1723.
\newblock \doi{10.1038/s41467-022-29443-w}.
\newblock URL \url{https://www.nature.com/articles/s41467-022-29443-w}.
\newblock Number: 1 Publisher: Nature Publishing Group.

\bibitem[Dubach et~al.(2017)Dubach, Kim, Yang, Cuccarese, Giedt, Meimetis,
  Vinegoni, and Weissleder]{dubach_quantitating_2017}
Dubach, J.~M., Kim, E., Yang, K., Cuccarese, M., Giedt, R.~J., Meimetis, L.~G.,
  Vinegoni, C., and Weissleder, R.
\newblock Quantitating drug-target engagement in single cells in vitro and in
  vivo.
\newblock \emph{Nature chemical biology}, 13\penalty0 (2):\penalty0 168--173,
  2017.
\newblock Publisher: Nature Publishing Group US New York.

\bibitem[Durbin et~al.(1998)Durbin, Eddy, Krogh, and
  Mitchison]{durbin_biological_1998}
Durbin, R., Eddy, S.~R., Krogh, A., and Mitchison, G.
\newblock \emph{Biological {Sequence} {Analysis}: {Probabilistic} {Models} of
  {Proteins} and {Nucleic} {Acids}}.
\newblock Cambridge University Press, 1 edition, April 1998.
\newblock ISBN 978-0-521-62041-3 978-0-521-62971-3 978-0-511-79049-2.
\newblock \doi{10.1017/CBO9780511790492}.
\newblock URL
  \url{https://www.cambridge.org/core/product/identifier/9780511790492/type/book}.

\bibitem[Eddy(2011)]{eddy_accelerated_2011}
Eddy, S.~R.
\newblock Accelerated {Profile} {HMM} {Searches}.
\newblock \emph{PLOS Computational Biology}, 7\penalty0 (10):\penalty0
  e1002195, October 2011.
\newblock ISSN 1553-7358.
\newblock \doi{10.1371/journal.pcbi.1002195}.
\newblock URL
  \url{https://journals.plos.org/ploscompbiol/article?id=10.1371/journal.pcbi.1002195}.
\newblock Publisher: Public Library of Science.

\bibitem[Feragen et~al.(2015)Feragen, Lauze, and
  Hauberg]{feragen_geodesic_2015}
Feragen, A., Lauze, F., and Hauberg, S.
\newblock Geodesic {Exponential} {Kernels}: {When} {Curvature} and {Linearity}
  {Conflict}.
\newblock pp.\  3032--3042, 2015.
\newblock URL
  \url{https://www.cv-foundation.org/openaccess/content_cvpr_2015/html/Feragen_Geodesic_Exponential_Kernels_2015_CVPR_paper.html}.

\bibitem[Frazer et~al.(2021)Frazer, Notin, Dias, Gomez, Min, Brock, Gal, and
  Marks]{frazer_disease_2021}
Frazer, J., Notin, P., Dias, M., Gomez, A., Min, J.~K., Brock, K., Gal, Y., and
  Marks, D.~S.
\newblock Disease variant prediction with deep generative models of
  evolutionary data.
\newblock \emph{Nature}, 599\penalty0 (7883):\penalty0 91--95, November 2021.
\newblock ISSN 1476-4687.
\newblock \doi{10.1038/s41586-021-04043-8}.
\newblock URL \url{https://www.nature.com/articles/s41586-021-04043-8}.
\newblock Number: 7883 Publisher: Nature Publishing Group.

\bibitem[Gao et~al.(2022)Gao, Fu, Sun, and Coley]{gao_sample_2022}
Gao, W., Fu, T., Sun, J., and Coley, C.~W.
\newblock Sample {Efficiency} {Matters}: {A} {Benchmark} for {Practical}
  {Molecular} {Optimization}, October 2022.
\newblock URL \url{http://arxiv.org/abs/2206.12411}.
\newblock arXiv:2206.12411 [cs, q-bio].

\bibitem[Garnett(2022)]{garnett_bayesian_2022}
Garnett, R.
\newblock \emph{Bayesian {Optimization}}.
\newblock Cambridge University Press, January 2022.

\bibitem[Garrido-Merchán \& Hernández-Lobato(2020)Garrido-Merchán and
  Hernández-Lobato]{garrido-merchan_dealing_2020}
Garrido-Merchán, E.~C. and Hernández-Lobato, D.
\newblock Dealing with categorical and integer-valued variables in {Bayesian}
  {Optimization} with {Gaussian} processes.
\newblock \emph{Neurocomputing}, 380:\penalty0 20--35, March 2020.
\newblock ISSN 0925-2312.
\newblock \doi{10.1016/j.neucom.2019.11.004}.
\newblock URL
  \url{https://www.sciencedirect.com/science/article/pii/S0925231219315619}.

\bibitem[Gruver et~al.(2023)Gruver, Stanton, Frey, Rudner, Hotzel,
  Lafrance-Vanasse, Rajpal, Cho, and Wilson]{gruver_protein_2023}
Gruver, N., Stanton, S., Frey, N.~C., Rudner, T. G.~J., Hotzel, I.,
  Lafrance-Vanasse, J., Rajpal, A., Cho, K., and Wilson, A.~G.
\newblock Protein {Design} with {Guided} {Discrete} {Diffusion}, May 2023.
\newblock URL \url{http://arxiv.org/abs/2305.20009}.
\newblock arXiv:2305.20009 [cs, q-bio].

\bibitem[Gómez-Bombarelli et~al.(2018)Gómez-Bombarelli, Wei, Duvenaud,
  Hernández-Lobato, Sánchez-Lengeling, Sheberla, Aguilera-Iparraguirre,
  Hirzel, Adams, and Aspuru-Guzik]{gomez-bombarelli_automatic_2018}
Gómez-Bombarelli, R., Wei, J.~N., Duvenaud, D., Hernández-Lobato, J.~M.,
  Sánchez-Lengeling, B., Sheberla, D., Aguilera-Iparraguirre, J., Hirzel,
  T.~D., Adams, R.~P., and Aspuru-Guzik, A.
\newblock Automatic {Chemical} {Design} {Using} a {Data}-{Driven} {Continuous}
  {Representation} of {Molecules}.
\newblock \emph{ACS Central Science}, 4\penalty0 (2):\penalty0 268--276,
  February 2018.
\newblock ISSN 2374-7943.
\newblock \doi{10.1021/acscentsci.7b00572}.
\newblock URL \url{https://doi.org/10.1021/acscentsci.7b00572}.
\newblock Publisher: American Chemical Society.

\bibitem[Harandi et~al.(2015)Harandi, Salzmann, and
  Baktashmotlagh]{harandi_beyond_2015}
Harandi, M., Salzmann, M., and Baktashmotlagh, M.
\newblock Beyond {Gauss}: {Image}-{Set} {Matching} on the {Riemannian}
  {Manifold} of {PDFs}.
\newblock In \emph{Proceedings of the {IEEE} {International} {Conference} on
  {Computer} {Vision} ({ICCV})}, December 2015.

\bibitem[Hellinger(1909)]{hellinger_neue_1909}
Hellinger, E.
\newblock Neue {Begründung} der {Theorie} quadratischer {Formen} von
  unendlichvielen {Veränderlichen}.
\newblock \emph{Journal für die reine und angewandte Mathematik},
  1909\penalty0 (136):\penalty0 210--271, July 1909.
\newblock ISSN 1435-5345.
\newblock \doi{10.1515/crll.1909.136.210}.
\newblock URL
  \url{https://www.degruyter.com/document/doi/10.1515/crll.1909.136.210/html}.
\newblock Publisher: De Gruyter.

\bibitem[Jones et~al.(1998)Jones, Schonlau, and Welch]{jones_efficient_1998}
Jones, D.~R., Schonlau, M., and Welch, W.~J.
\newblock Efficient {Global} {Optimization} of {Expensive} {Black}-{Box}
  {Functions}.
\newblock \emph{Journal of Global Optimization}, 13\penalty0 (4):\penalty0
  455--492, December 1998.
\newblock ISSN 1573-2916.
\newblock \doi{10.1023/A:1008306431147}.
\newblock URL \url{https://doi.org/10.1023/A:1008306431147}.

\bibitem[Kingma \& Welling(2019)Kingma and Welling]{kingma_introduction_2019}
Kingma, D.~P. and Welling, M.
\newblock An {Introduction} to {Variational} {Autoencoders}.
\newblock \emph{Foundations and Trends® in Machine Learning}, 12\penalty0
  (4):\penalty0 307--392, 2019.
\newblock ISSN 1935-8237, 1935-8245.
\newblock \doi{10.1561/2200000056}.
\newblock URL \url{http://arxiv.org/abs/1906.02691}.
\newblock arXiv:1906.02691 [cs, stat].

\bibitem[Lu et~al.(2018)Lu, Gonzalez, Dai, and Lawrence]{lu_structured_2018}
Lu, X., Gonzalez, J., Dai, Z., and Lawrence, N.~D.
\newblock Structured {Variationally} {Auto}-encoded {Optimization}.
\newblock In \emph{Proceedings of the 35th {International} {Conference} on
  {Machine} {Learning}}, pp.\  3267--3275. PMLR, July 2018.
\newblock URL \url{https://proceedings.mlr.press/v80/lu18c.html}.
\newblock ISSN: 2640-3498.

\bibitem[{Martín Abadi} et~al.(2015){Martín Abadi}, {Ashish Agarwal}, {Paul
  Barham}, {Eugene Brevdo}, {Zhifeng Chen}, {Craig Citro}, {Greg S. Corrado},
  {Andy Davis}, {Jeffrey Dean}, {Matthieu Devin}, {Sanjay Ghemawat}, {Ian
  Goodfellow}, {Andrew Harp}, {Geoffrey Irving}, {Michael Isard}, Jia, {Rafal
  Jozefowicz}, {Lukasz Kaiser}, {Manjunath Kudlur}, {Josh Levenberg},
  {Dandelion Mané}, {Rajat Monga}, {Sherry Moore}, {Derek Murray}, {Chris
  Olah}, {Mike Schuster}, {Jonathon Shlens}, {Benoit Steiner}, {Ilya
  Sutskever}, {Kunal Talwar}, {Paul Tucker}, {Vincent Vanhoucke}, {Vijay
  Vasudevan}, {Fernanda Viégas}, {Oriol Vinyals}, {Pete Warden}, {Martin
  Wattenberg}, {Martin Wicke}, {Yuan Yu}, and {Xiaoqiang
  Zheng}]{martin_abadi_tensorflow_2015}
{Martín Abadi}, {Ashish Agarwal}, {Paul Barham}, {Eugene Brevdo}, {Zhifeng
  Chen}, {Craig Citro}, {Greg S. Corrado}, {Andy Davis}, {Jeffrey Dean},
  {Matthieu Devin}, {Sanjay Ghemawat}, {Ian Goodfellow}, {Andrew Harp},
  {Geoffrey Irving}, {Michael Isard}, Jia, Y., {Rafal Jozefowicz}, {Lukasz
  Kaiser}, {Manjunath Kudlur}, {Josh Levenberg}, {Dandelion Mané}, {Rajat
  Monga}, {Sherry Moore}, {Derek Murray}, {Chris Olah}, {Mike Schuster},
  {Jonathon Shlens}, {Benoit Steiner}, {Ilya Sutskever}, {Kunal Talwar}, {Paul
  Tucker}, {Vincent Vanhoucke}, {Vijay Vasudevan}, {Fernanda Viégas}, {Oriol
  Vinyals}, {Pete Warden}, {Martin Wattenberg}, {Martin Wicke}, {Yuan Yu}, and
  {Xiaoqiang Zheng}.
\newblock {TensorFlow}: {Large}-{Scale} {Machine} {Learning} on {Heterogeneous}
  {Systems}, 2015.
\newblock URL \url{https://www.tensorflow.org/}.

\bibitem[Matthews et~al.(2017)Matthews, van~der Wilk, Nickson, Fujii,
  Boukouvalas, León-Villagrá, Ghahramani, and Hensman]{matthews_gpflow_2017}
Matthews, A. G. d.~G., van~der Wilk, M., Nickson, T., Fujii, K., Boukouvalas,
  A., León-Villagrá, P., Ghahramani, Z., and Hensman, J.
\newblock {GPflow}: {A} {Gaussian} process library using {TensorFlow}.
\newblock \emph{Journal of Machine Learning Research}, 18\penalty0
  (40):\penalty0 1--6, April 2017.
\newblock URL \url{http://jmlr.org/papers/v18/16-537.html}.

\bibitem[Maus et~al.(2023)Maus, Wu, Eriksson, and
  Gardner]{maus_discovering_2023}
Maus, N., Wu, K., Eriksson, D., and Gardner, J.
\newblock Discovering {Many} {Diverse} {Solutions} with {Bayesian}
  {Optimization}, May 2023.
\newblock URL \url{http://arxiv.org/abs/2210.10953}.
\newblock arXiv:2210.10953 [cs].

\bibitem[Maynard~Smith(1970)]{maynard_smith_natural_1970}
Maynard~Smith, J.
\newblock Natural {Selection} and the {Concept} of a {Protein} {Space}.
\newblock \emph{Nature}, 225\penalty0 (5232):\penalty0 563--564, February 1970.
\newblock ISSN 1476-4687.
\newblock \doi{10.1038/225563a0}.
\newblock URL \url{https://www.nature.com/articles/225563a0}.
\newblock Number: 5232 Publisher: Nature Publishing Group.

\bibitem[Moss et~al.(2020)Moss, Leslie, Beck, Gonzalez, and
  Rayson]{moss_boss_2020}
Moss, H., Leslie, D., Beck, D., Gonzalez, J., and Rayson, P.
\newblock Boss: {Bayesian} optimization over string spaces.
\newblock \emph{Advances in neural information processing systems},
  33:\penalty0 15476--15486, 2020.

\bibitem[Močkus(1975)]{mockus_bayesian_1975}
Močkus, J.
\newblock On {Bayesian} {Methods} for {Seeking} the {Extremum}.
\newblock In Marchuk, G.~I. (ed.), \emph{Optimization {Techniques} {IFIP}
  {Technical} {Conference}: {Novosibirsk}, {July} 1–7, 1974}, Lecture {Notes}
  in {Computer} {Science}, pp.\  400--404. Springer, Berlin, Heidelberg, 1975.
\newblock ISBN 978-3-662-38527-2.
\newblock \doi{10.1007/978-3-662-38527-2_55}.
\newblock URL \url{https://doi.org/10.1007/978-3-662-38527-2_55}.

\bibitem[Notin et~al.(2022)Notin, Van~Niekerk, Kollasch, Ritter, Gal, and
  Marks]{notin_trancepteve_2022}
Notin, P.~M., Van~Niekerk, L., Kollasch, A.~W., Ritter, D., Gal, Y., and Marks,
  D.
\newblock {TranceptEVE}: {Combining} {Family}-specific and {Family}-agnostic
  {Models} of {Protein} {Sequences} for {Improved} {Fitness} {Prediction}.
\newblock preprint, Genetics, December 2022.
\newblock URL \url{http://biorxiv.org/lookup/doi/10.1101/2022.12.07.519495}.

\bibitem[Picheny et~al.(2023)Picheny, Berkeley, Moss, Stojic, Granta, Ober,
  Artemev, Ghani, Goodall, Paleyes, Vakili, Pascual-Diaz, Markou, Qing, Loka,
  and Couckuyt]{picheny_trieste_2023}
Picheny, V., Berkeley, J., Moss, H.~B., Stojic, H., Granta, U., Ober, S.~W.,
  Artemev, A., Ghani, K., Goodall, A., Paleyes, A., Vakili, S., Pascual-Diaz,
  S., Markou, S., Qing, J., Loka, N. R. B.~S., and Couckuyt, I.
\newblock Trieste: {Efficiently} {Exploring} {The} {Depths} of {Black}-box
  {Functions} with {TensorFlow}, 2023.
\newblock URL \url{https://arxiv.org/abs/2302.08436}.

\bibitem[Rasmussen \& Williams(2006)Rasmussen and
  Williams]{rasmussen_gaussian_2006}
Rasmussen, C.~E. and Williams, C. K.~I.
\newblock \emph{Gaussian processes for machine learning}.
\newblock Adaptive computation and machine learning. MIT Press, Cambridge,
  Mass, 2006.
\newblock ISBN 978-0-262-18253-9.
\newblock OCLC: ocm61285753.

\bibitem[Riesselman et~al.(2018)Riesselman, Ingraham, and
  Marks]{riesselman_deep_2018}
Riesselman, A.~J., Ingraham, J.~B., and Marks, D.~S.
\newblock Deep generative models of genetic variation capture the effects of
  mutations.
\newblock \emph{Nature Methods}, 15\penalty0 (10):\penalty0 816--822, October
  2018.
\newblock ISSN 1548-7105.
\newblock \doi{10.1038/s41592-018-0138-4}.
\newblock URL \url{https://www.nature.com/articles/s41592-018-0138-4}.
\newblock Number: 10 Publisher: Nature Publishing Group.

\bibitem[Shahriari et~al.(2016)Shahriari, Swersky, Wang, Adams, and
  De~Freitas]{shahriari_taking_2016}
Shahriari, B., Swersky, K., Wang, Z., Adams, R.~P., and De~Freitas, N.
\newblock Taking the {Human} {Out} of the {Loop}: {A} {Review} of {Bayesian}
  {Optimization}.
\newblock \emph{Proceedings of the IEEE}, 104\penalty0 (1):\penalty0 148--175,
  January 2016.
\newblock ISSN 0018-9219, 1558-2256.
\newblock \doi{10.1109/JPROC.2015.2494218}.
\newblock URL \url{https://ieeexplore.ieee.org/document/7352306/}.

\bibitem[Srinivas et~al.(2012)Srinivas, Krause, Kakade, and
  Seeger]{srinivas_gaussian_2012}
Srinivas, N., Krause, A., Kakade, S.~M., and Seeger, M.
\newblock Gaussian {Process} {Optimization} in the {Bandit} {Setting}: {No}
  {Regret} and {Experimental} {Design}.
\newblock \emph{IEEE Transactions on Information Theory}, 58\penalty0
  (5):\penalty0 3250--3265, May 2012.
\newblock ISSN 0018-9448, 1557-9654.
\newblock \doi{10.1109/TIT.2011.2182033}.
\newblock URL \url{http://arxiv.org/abs/0912.3995}.
\newblock arXiv:0912.3995 [cs].

\bibitem[Stanton et~al.(2022)Stanton, Maddox, Gruver, Maffettone, Delaney,
  Greenside, and Wilson]{stanton_accelerating_2022}
Stanton, S., Maddox, W., Gruver, N., Maffettone, P., Delaney, E., Greenside,
  P., and Wilson, A.~G.
\newblock Accelerating {Bayesian} {Optimization} for {Biological} {Sequence}
  {Design} with {Denoising} {Autoencoders}.
\newblock Technical Report arXiv:2203.12742, arXiv, July 2022.
\newblock URL \url{http://arxiv.org/abs/2203.12742}.
\newblock arXiv:2203.12742 [cs, q-bio, stat] type: article.

\bibitem[Swersky et~al.(2013)Swersky, Snoek, and
  Adams]{swersky_multi-task_2013}
Swersky, K., Snoek, J., and Adams, R.~P.
\newblock Multi-{Task} {Bayesian} {Optimization}.
\newblock In \emph{Advances in {Neural} {Information} {Processing} {Systems}},
  volume~26. Curran Associates, Inc., 2013.
\newblock URL
  \url{https://proceedings.neurips.cc/paper/2013/hash/f33ba15effa5c10e873bf3842afb46a6-Abstract.html}.

\bibitem[Swersky et~al.(2020)Swersky, Rubanova, Dohan, and
  Murphy]{swersky_amortized_2020}
Swersky, K., Rubanova, Y., Dohan, D., and Murphy, K.
\newblock Amortized {Bayesian} {Optimization} over {Discrete} {Spaces}.
\newblock In \emph{Proceedings of the 36th {Conference} on {Uncertainty} in
  {Artificial} {Intelligence} ({UAI})}, pp.\  769--778. PMLR, August 2020.
\newblock URL \url{https://proceedings.mlr.press/v124/swersky20a.html}.
\newblock ISSN: 2640-3498.

\bibitem[Tian \& Best(2017)Tian and Best]{tian_how_2017}
Tian, P. and Best, R.~B.
\newblock How {Many} {Protein} {Sequences} {Fold} to a {Given} {Structure}? {A}
  {Coevolutionary} {Analysis}.
\newblock \emph{Biophysical Journal}, 113\penalty0 (8):\penalty0 1719--1730,
  October 2017.
\newblock ISSN 00063495.
\newblock \doi{10.1016/j.bpj.2017.08.039}.
\newblock URL
  \url{https://linkinghub.elsevier.com/retrieve/pii/S0006349517309347}.

\bibitem[Tripp et~al.(2020)Tripp, Daxberger, and
  Hernández-Lobato]{tripp_sample-efficient_2020}
Tripp, A., Daxberger, E., and Hernández-Lobato, J.~M.
\newblock Sample-{Efficient} {Optimization} in the {Latent} {Space} of {Deep}
  {Generative} {Models} via {Weighted} {Retraining}.
\newblock In \emph{Advances in {Neural} {Information} {Processing} {Systems}},
  volume~33, pp.\  11259--11272. Curran Associates, Inc., 2020.
\newblock URL
  \url{https://proceedings.neurips.cc/paper/2020/hash/81e3225c6ad49623167a4309eb4b2e75-Abstract.html}.

\bibitem[Virtanen et~al.(2020)Virtanen, Gommers, Oliphant, Haberland, Reddy,
  Cournapeau, Burovski, Peterson, Weckesser, Bright, van~der Walt, Brett,
  Wilson, Millman, Mayorov, Nelson, Jones, Kern, Larson, Carey, Polat, Feng,
  Moore, VanderPlas, Laxalde, Perktold, Cimrman, Henriksen, Quintero, Harris,
  Archibald, Ribeiro, Pedregosa, van Mulbregt, and {SciPy 1.0
  Contributors}]{virtanen_scipy_2020}
Virtanen, P., Gommers, R., Oliphant, T.~E., Haberland, M., Reddy, T.,
  Cournapeau, D., Burovski, E., Peterson, P., Weckesser, W., Bright, J.,
  van~der Walt, S.~J., Brett, M., Wilson, J., Millman, K.~J., Mayorov, N.,
  Nelson, A. R.~J., Jones, E., Kern, R., Larson, E., Carey, C.~J., Polat, I.,
  Feng, Y., Moore, E.~W., VanderPlas, J., Laxalde, D., Perktold, J., Cimrman,
  R., Henriksen, I., Quintero, E.~A., Harris, C.~R., Archibald, A.~M., Ribeiro,
  A.~H., Pedregosa, F., van Mulbregt, P., and {SciPy 1.0 Contributors}.
\newblock {SciPy} 1.0: {Fundamental} {Algorithms} for {Scientific} {Computing}
  in {Python}.
\newblock \emph{Nature Methods}, 17:\penalty0 261--272, 2020.
\newblock \doi{10.1038/s41592-019-0686-2}.

\bibitem[Wan et~al.(2021)Wan, Nguyen, Ha, Ru, Lu, and Osborne]{wan_think_2021}
Wan, X., Nguyen, V., Ha, H., Ru, B., Lu, C., and Osborne, M.~A.
\newblock Think {Global} and {Act} {Local}: {Bayesian} {Optimisation} over
  {High}-{Dimensional} {Categorical} and {Mixed} {Search} {Spaces}.
\newblock In \emph{Proceedings of the 38th {International} {Conference} on
  {Machine} {Learning}}, pp.\  10663--10674. PMLR, July 2021.
\newblock URL \url{https://proceedings.mlr.press/v139/wan21b.html}.
\newblock ISSN: 2640-3498.

\end{thebibliography}
\bibliographystyle{icml2024}


\appendix

\newpage
\appendix
\onecolumn
\section{Proof that the relaxed objective has the same optima}
\label{sec:same_optima}
\begin{proposition}
\label{thm:same_optima}
For $f$ and $\bar{f}$: $\argmin f=\argmin \bar{f}$.
\end{proposition}
\begin{proof}
\begin{align*}
	\bar{f}(\vec p)&=\sum_{\vec x}f(\vec x)p(\vec x)
	\\&\leq \sum_{\vec x}\max_{\vec x'}f(\vec x')p(\vec x)
	\\&=\max_{\vec x'}f(\vec x') \sum_{\vec x}p(\vec x)
	\\&=\max_{\vec x'}f(\vec x')
\end{align*}

On the other hand: for an optimal $\vec x_*$, \it{i.e.} $f(\vec x_*)=\max_{\vec x'}f(\vec x')$ and choose $\vec p_*(\vec x)\ce \mathbf{1}_{\vec x=\vec x_*}(\vec x)$, then
$\bar{f}(\vec p_*)=\sum_{\vec x}f(\vec x)p_*(\vec x)f(\vec x)=\max_{\vec x'}f(\vec x')$.
\end{proof}

\section{The weighted Hellinger distance}
\begin{proposition}
\label{thm:weighted_hd}
The squared weighted Hellinger distance 
$$HD_{\vec p}(\vec q, \vec r)^2\ce \frac{1}{2}\sum_{\vec x}\vec p(\vec x)\left(\sqrt{\vec q(\vec x)}-\sqrt{\vec r(\vec x)}\right)^2$$
is negative definite.
\end{proposition}
\begin{proof}
The proof follows \citet{harandi_beyond_2015}. 
Let $N\in\mathbb{N}$ and $c_1,\dots,c_N\in\Re$ s.t.~$\sum_{n=1}^N c_n=0$.
\begin{align}
&\sum_{n,m=1}^N c_n c_m HD_{\vec p}(\vec q_n, \vec q_m)^2
\\&=\frac{1}{2}\sum_{n,m=1}^N c_n c_m\sum_{\vec x\in\mathbb{X}}\vec p(\vec x)\left(\sqrt{\vec q_n(\vec x)}-\sqrt{\vec q_m(\vec x)}\right)^2
\eqcomment{by definition}
\\&=\frac{1}{2}\sum_{n,m=1}^N c_n c_m\sum_{\vec x\in\mathbb{X}}\vec p(\vec x)\left(\vec q_n(\vec x)+\vec q_m(\vec x)-2\sqrt{\vec q_n(\vec x)\vec q_m(\vec x)} \right)\eqcomment{expanding the square}
\\&= \sum_{\vec x\in\mathbb{X}}\vec p(\vec x)\left(\sum_{n=1}^N c_n \vec q_n(\vec x)\sum_{m=1}^N c_m+\sum_{m=1}^N c_m \vec q_m(\vec x)\sum_{n=1}^N c_n - \sum_{n=1}^N c_n \sqrt{\vec q_n(\vec x)} \sum_{m=1}^N c_m\sqrt{\vec q_m(\vec x)}\right) 
\eqcomment{changing order of summation}
\\&= -\sum_{\vec x\in\mathbb{X}}\vec p(\vec x)\sum_{n=1}^N c_n \sqrt{\vec q_n(\vec x)} \sum_{m=1}^N c_m\sqrt{\vec q_m(\vec x)} 
\eqcomment{since $\sum_{n=1}^N c_n=0$}
\\&= -\sum_{\vec x\in\mathbb{X}}\vec p(\vec x)\left(\sum_{n=1}^N c_n \sqrt{\vec q_n(\vec x)}\right)^2 
\eqcomment{writing the identical sums over $n$ and $m$ as a square}
\\&\leq 0
\end{align}
\end{proof}

To show that the square-root of the tilted Hellinger distance is a kernel, we follow the same reasoning as in \citet{harandi_beyond_2015}.

\section{Efficient evaluation of the Hellinger distance}
\label{sec:efficient_hd}
\begin{proposition}
\label{thm:efficient_hd}
For product measures $\vec p, \vec q\in \mathbb{P}_f$, the Hellinger distance can be written as 
$$HD(\vec p, \vec q)^2=1-\prod_{l=1}^L\sum_{a_l=1}^A \sqrt{\vec p[a_l, l]\vec q[a_l, l]}$$
\end{proposition}
\begin{proof}
\begin{align}
HD(\vec p, \vec q)^2&=\frac{1}{2}\sum_{\vec x\in\mathbb{X}}\left(\sqrt{\vec p(\vec x)}-\sqrt{\vec q(\vec x)}\right)^2 
\eqcomment{expanding the square and using that $\vec p$ and $\vec q$ sum to 1.}
\\&=1-\sum_{\vec x\in\mathbb{X}}\sqrt{\vec p(\vec x)\vec q(\vec x)} \eqcomment{property of the Hellinger distance}
\\&=1-\underbrace{\sum_{a=1}^A \dots \sum_{a=1}^A}_{L \text{ times}}\sqrt{\vec p(a_1, \dots, a_L)\vec q(a_1, \dots, a_L)}
\eqcomment{rewriting the sum}
\\&=1-\underbrace{\sum_{a_1=1}^A \dots \sum_{a_L=1}^A}_{L \text{ times}}\sqrt{\prod_{l=1}^L\vec p[a_l, l]\vec q[a_l, l]}
\eqcomment{using $\vec p, \vec q \in \mathbb{P}_f$}
\\&=1-\underbrace{\sum_{a_1=1}^A \dots \sum_{a_L=1}^A}_{L \text{ times}}\prod_{l=1}^L\sqrt{\vec p[a_l, l]\vec q[a_l, l]}
\eqcomment{moving the square-root}
\\&=1-\sum_{a_1=1}^A \sqrt{\vec p[a_1, 1]\vec q[a_1, 1]}\ \dot \ \dots\ \dot \  \sum_{a_L=1}^A\sqrt{\vec p[a_L, L]\vec q[a_L, L]}
\eqcomment{moving unaffected parts of the product out of the sum}
\\&=1-\prod_{l=1}^L\sum_{a_l=1}^A \sqrt{\vec p[a_l, l]\vec q[a_l, l]}
\eqcomment{rearranging}
\end{align}
\end{proof}

\begin{proposition}
For product measures $\vec p, \vec q, \vec r\in \mathbb{P}_f$, the weighted Hellinger distance can be written as 
\begin{align*}
   HD(\vec p, \vec q, \vec r)^2 &= \prod_{l=1}^L\sum_{a_l=1}^A \left[\frac{1}{2} \vec r[a_l,l] \vec p[a_l, l]
+ \frac{1}{2}  \vec r[a_l,l] \vec q[a_l, l]\right]
-\prod_{l=1}^L\sum_{a_l=1}^A \vec r[a_l, l] \sqrt{\vec p[a_l, l]\vec q[a_l, l]} \\
&= \prod_{l=1}^L\sum_{a_l=1}^A \frac{1}{2}\mathbb{E}_p[\vec x]\vec r[a_l, l] + \frac{1}{2} \mathbb{E}_q[\vec x]\vec r[a_l, l]
- \prod_{l=1}^L\sum_{a_l=1}^A \vec r[a_l, l]\sqrt{\vec p[a_l, l]\vec q[a_l, l]}. 
\end{align*}
\end{proposition}

\begin{proof}
\begin{align}
HD_{\vec r}(\vec p, \vec q)^2
&=\frac{1}{2}\sum_{\vec x\in\mathbb{X}}r(\vec x)\left(\sqrt{\vec p(\vec x)}-\sqrt{\vec q(\vec x)}\right)^2 
\\&=\frac{1}{2}\sum_{\vec x\in\mathbb{X}}r(\vec x)\left(\vec p(\vec x)-2\sqrt{\vec p(x)\vec q(\vec x)}+q(\vec x
)\right) 
\eqcomment{expanding the square}
\\&=\frac{1}{2}\sum_{\vec x\in\mathbb{X}}\left[\vec r(\vec x )\vec p(\vec x) 
- 2 \vec r(\vec x)\sqrt{\vec p(\vec x)\vec q(\vec x)} 
+ \vec r(\vec x) \vec q(\vec x)\right]
\eqcomment{note that $\sum_{\vec x\in\mathbb{X}} \vec r(\vec x) \vec p(\vec x) \neq 1$}
\\&=\frac{1}{2}\sum_{\vec x\in\mathbb{X}} \vec r(\vec x) \vec p(\vec x) 
+ \frac{1}{2}\sum_{\vec x\in\mathbb{X}} \vec r(\vec x) \vec q(\vec x)  
- \sum_{\vec x\in\mathbb{X}}\vec r(\vec x)\sqrt{\vec p(\vec x)\vec q(\vec x)} 
\eqcomment{property of the Hellinger distance}
\\&=
\frac{1}{2}\underbrace{\sum_{a_1=1}^A \dots \sum_{a_L=1}^A}_{L \text{ times}} \vec r(x_1, \dots, x_L) \vec p(x_1, \dots, x_L)+ \frac{1}{2}\underbrace{\sum_{a_1=1}^A \dots \sum_{a_L=1}^A}_{L \text{ times}} \vec r(x_1, \dots, x_L) \vec q(x_1, \dots, x_L)\\ 
&\ - \underbrace{\sum_{a_1=1}^A \dots \sum_{a_L=1}^A}_{L \text{ times}}\vec r(x_1, \dots, x_L)\sqrt{\vec p(x_1, \dots, x_L)\vec q(x_1, \dots, x_L)}
\eqcomment{factorize}
\\&=
\frac{1}{2}\underbrace{\sum_{a_1=1}^A \dots \sum_{a_L=1}^A}_{L \text{ times}}\prod_{l=1}^L \vec r[a_l, l] \vec p[a_l, l]+ \frac{1}{2}\underbrace{\sum_{a_1=1}^A \dots \sum_{a_L=1}^A}_{L \text{ times}}\prod_{l=1}^L \vec r[a_l, l] \vec q[a_l, l]
\\&\ - \underbrace{\sum_{a_1=1}^A \dots \sum_{a_L=1}^A}_{L \text{ times}}\prod_{l=1}^L\vec r[a_l, l]\sqrt{\vec p[a_l, l]\vec q[a_l, l]}
\eqcomment{rearrange, and sums of products as products of sums - see \ref{sec:efficient_hd}}
\\&=
\prod_{l=1}^L \sum_{a_l=1}^A \frac{1}{2} \vec r[a_l, l] \vec p[a_l, l] + \frac{1}{2} \vec r[a_l, l] \vec q[a_l, l]
- \prod_{l=1}^L \sum_{a_l=1}^A \vec r[a_l, l]\sqrt{\vec p[a_l, l]\vec q[a_l, l]}
\end{align}
\end{proof}

\newpage

\section{\CoRel{} algorithm specifications}
\subsection{A continuous optimization algorithm}
\begin{algorithm}
\caption{\CoRel{} using continuous optimization}
\begin{algorithmic}
\Require acquisition $a:\factorP\rightarrow\Re $, black-box $f:\mathbb{X}\rightarrow \Re $, dataset $\mathcal{D}_1=\{X, y\}$, pretrained LVM $\phi: \Re^D\rightarrow \factorP$
\For{$t \in 1, ..., t_{\max}$}
\State $m\gets$ trainModel($(\mathbbm{1}_{\vec x_i}, \vec y_i)_{i=1}^t$) 
\State $\vec p_* \gets \arg\max_{\vec p}a(\vec p, m)$ 
\State $\vec x\gets$ getSequenceFromDistribution($\vec p_*$)
\State $\mathcal{D}_{t+1} \gets \mathcal{D}_{t} \cup \{\vec x, f(\vec x)\}$
\EndFor
\end{algorithmic}
\label{algo:si:continuous_bo}
\end{algorithm}
\subsection{A discrete optimization algorithm}

\begin{algorithm}
\caption{\CoRel{} using discrete optimization}
\begin{algorithmic}
\Require acquisition $a:\factorP\rightarrow\Re $, black-box $f:\mathbb{X}\rightarrow \Re $, dataset $\mathcal{D}_1=\{X, y\}$, pretrained LVM $\phi: \Re^D\rightarrow \factorP$
\For{$t \in 1, ..., t_{\max}$}
\State $m\gets$ trainModel($(\mathbbm{1}_{\vec x_i}, \vec y_i)_{i=1}^t$) 
\State $\vec x \gets \arg\max_{\vec x'}a(\mathbbm{1}_{\vec x'}, m)$ 
\State $\mathcal{D}_{t+1} \gets \mathcal{D}_{t} \cup \{\vec x, f(\vec x)\}$
\EndFor
\end{algorithmic}
\label{algo:si:discrete_bo}
\end{algorithm}

We require a function for finding pareto optimal points, given a set of all points $S$ with $x\subset S$: 
\begin{align}
\label{def:paretofront}
    p_\text{opt}({x}):=\{x\in S | \nexists  x'\in S \text{ s.t. } x'\preceq x \land x'\neq x \}.
\end{align}.
In our experiments pareto optimal points are determined by the $y$ vector.

\section{Baseline implementations and hyperparameters}

\paragraph{Implementation and optimization}
Models and experiments are implemented with Tensorflow \citep{martin_abadi_tensorflow_2015}, Tensorflow-probability, GPflow \citep{matthews_gpflow_2017}, and Trieste \citep{picheny_trieste_2023}. 
Model hyperparameters are optimized using the scipy LBFGS optimizer \citep{virtanen_scipy_2020} on the model likelihood as previously described.

\paragraph{Discrete biological sequence optimization library}
This library contains both the \RFP{} and \GFP{} problem as well as the stable \emph{LamBO} implementation for experimental queries.
We define the RFP problem in generally with FoldX and SASA computations, respective the LamBO defined pareto front.
We additionally include a reference objective that is equivalent to the LamBO setup and includes significantly more sequences.

\paragraph{\emph{CBas} VAE model}
\label{sec:cbas_setup}
The \emph{GFP} problem is presented in \citet{brookes_conditioning_2019} with a VAE as a latent variable model and a custom predictive GP model. 
The full training corpus are 54025 unique sequences for which observations are available.
For the VAE: the encoder is a simple neural network (size=50 units) with ELU activation mapping to 20 latent dimensions.
The decoder is a deep neural network with 3 layers of dimensions [50, 20*len(sequence), 20]
with ELU and softmax activation respectively, such that we first obtain a mapping from latent space to hidden space of size number of amino acids times (aligned) sequence length and ultimately the label-encoded protein sequence.
Model input are aligned GFP sequences.

Training data are 5000 sequences of 20 amino-acid tokens. 
The training data are the upper-quantile (by fluorescence measurement), which is approximately 9\% of all available sequences. 
Training has been done for 100 epochs in batches of 10 using a (tensorflow2) Adam optimizer with default configuration.

The model latent space is used in combination with a predictive GP oracle model as a surrogate for \emph{GFP}.
For a specification of the predictive GP model we refer to \cite{brookes_conditioning_2019}.

The optimization task for the \GFP{} problem is the minimization of the negative oracle values, to find the global minimum - see \cref{fig:continuous_opt}.
For later analysis the sign is inverted, as we ultimately intend to maximize the oracle.

We adapt this architecture to two latent dimensions for visualization purposes only. 
All empirical results queried against the GP utilize the initial model of full dimensionality.
The remaining model specifications remain the same.

\paragraph{LamBO latent optimizer}
\label{sec:lambo_setup}
We use the LamBO implementation directly from the stable tagged submission branch of \citet{stanton_accelerating_2022} (commit 22afec26da0b9ea1810e65f8a60ea7988c021cef Oct-2022) and refer to their exact model specifications.
The only addition we make is to define a dedicated task protocol to work with the discrete sequence optimization library we provide, and associated configurations.\\
We highlight that the original model specification and experiments account for a large range of starting sequences including seed-sequences. 
Additionally, when analyzing the model and observations we find that the relative hypervolume is computed respective a hard-coded set of reference values. 
For our benchmarking purposes we record this value also (see Supplementary \cref{fig:si:lambo_ref_HV}).
In our case we obtain the start values from the black-box function evaluations and compute the relative HV with respect to that. 
We provide the algorithm only with an exact set of 6 RFP PDB files to use, and do not access any provided seed-sequences.

The LamBO RFP data-set contains 50 PDB files (RFP structure files), 754 related sequences, and 1923 generated proxy seed sequences.

\paragraph{Random Mutations} As a naïve baseline, we consider mutating positions at random in the Pareto front. At each iteration, we maintain a population of 16 elements (padding with random mutations if the Pareto front is not large enough). Each of these elements is then mutated in two random positions (taking into account that sequences may have varying lengths, and never performing any inserting/deleting operations). The results provided were averaged over 5 iterations with different random seeds. The performance of this baseline is noteworthy, and the difference between it and e.g. \cite{stanton_accelerating_2022}'s NSGA-2 may be attributed to the fact that we mutate \textit{twice} instead of \textit{once} per iteration.
\newpage

\section{Additional Results}

\begin{figure}[ht]
    \hfill
        \includegraphics[width=0.9\columnwidth]{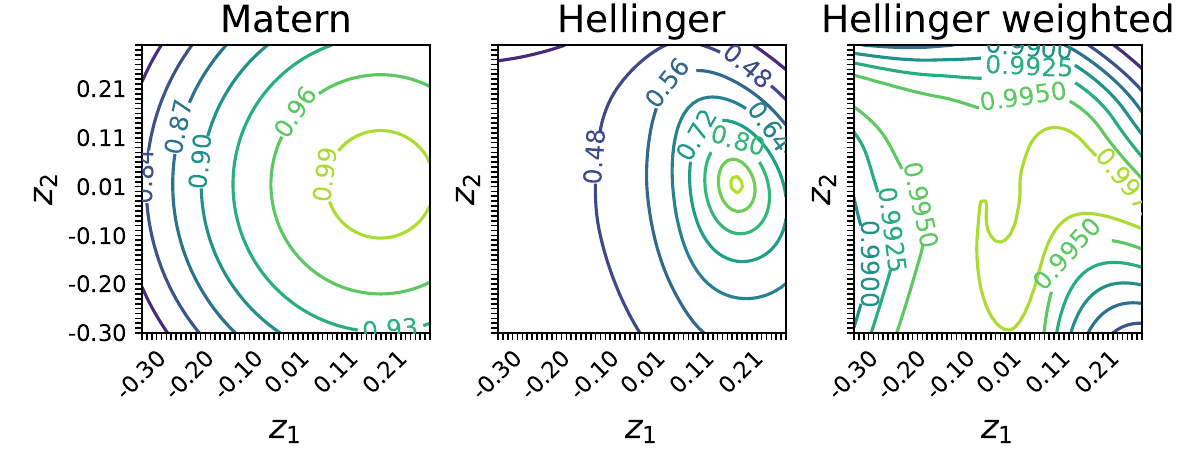}
        \caption{Covariance function values for the \GFP{} (2D) latent space where the reference point is the GFP wild-type sequence. Comparing Matérn 5/2 with the (weighted) Hellinger kernel.
        The reference points corresponds to the \emph{start} point in \cref{fig:latent_space}.
        }
        \label{fig:hellinger_contour_gfp}
\end{figure}
\vspace{1em}

\begin{figure}[ht]
    \centering
    \includegraphics[width=0.4\textwidth]{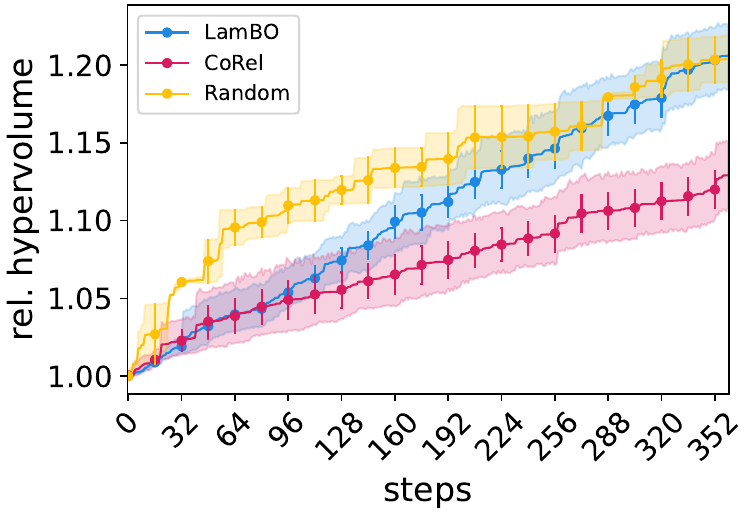} 
    \caption{Discretely optimizating \RFP{}. We compare against LamBO in the warm-start setting. Starting N=50, batch-size=16 across seven seeds (random two seeds).
    Markers indicate batch averages with std.err. bars. Shaded region is 95\% CI.
    }
    \label{fig:si:warm_HV}
\end{figure}

\begin{figure}
    \centering
    \includegraphics[width=0.4\textwidth]{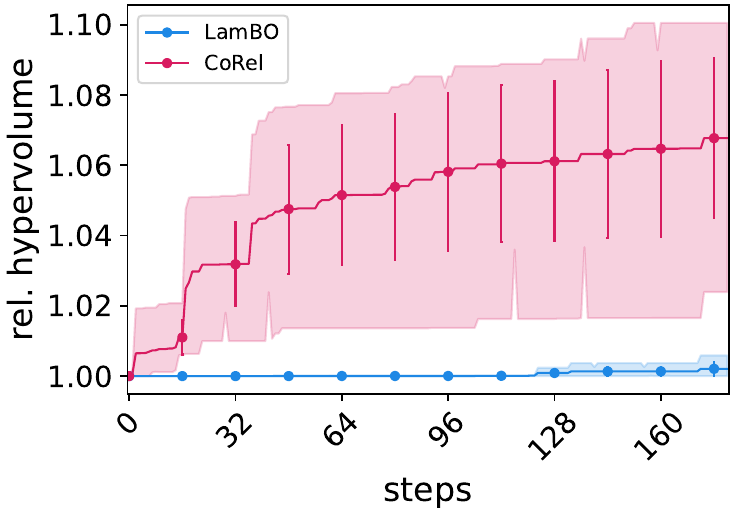}
    \caption{Optimizing \RFP{} discretely in the reference case with 512 starting sequences. We compare \CoRel{} against LamBO in the reference setup, batch-size=16 across three seeds.
    Given that we start with a relatively large starting hypervolume only marginal improvements can be achieved.
    Markers indicate batch averages with std.err. bars. Shaded region is 95\% CI.
    }
    \label{fig:si:ref_HV}
\end{figure}

\begin{figure}
    \centering
    \includegraphics[width=0.4\textwidth]{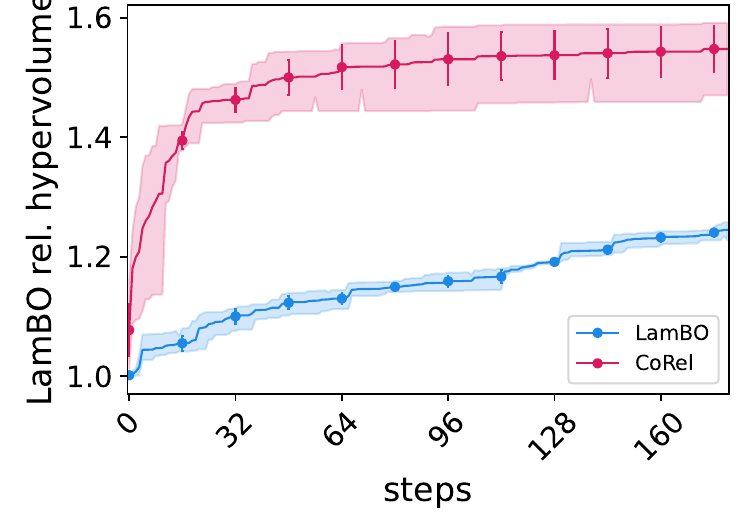}
    \caption{Optimizing \RFP{} discretely in the reference case with 512 starting sequences using the LamBO specific relative hypervolume improvement. 
    This metric is computed with internal reference Pareto front values, which remain fixed across all experiments and display a larger relative improvement over time. 
    Batch-size is 16 across three seeds.
    Markers indicate batch averages with std.err. bars. Shaded region is 95\% CI.
    }
    \label{fig:si:lambo_ref_HV}
\end{figure}

\begin{figure}
    \centering
    \includegraphics[width=0.4\textwidth]{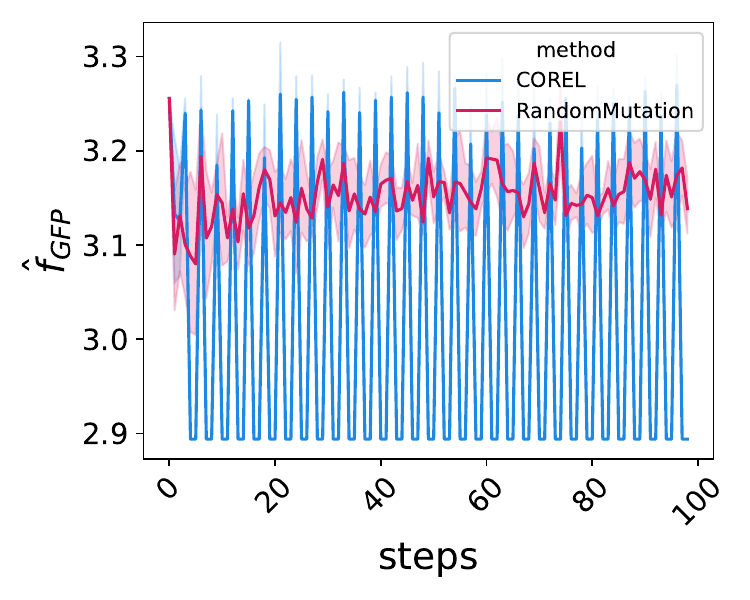}
    \caption{Oracle observations during the course of 100 GFP optimization steps. \CoRel{} jumps between extreme values in contrast to the random exploration.
    }
    \label{fig:si:gfp_obs}
\end{figure}

\end{document}